\documentclass[11pt]{article}
\pdfoutput=1
\usepackage{booktabs}
\usepackage{pgfplots}
\pgfplotsset{compat=1.14}
\usepackage{amsmath, amsthm}
\usepackage{amssymb}
\usepackage{graphicx}
\usepackage{setspace}
\usepackage{xfrac}
\usepackage{hyperref}
\usepackage{xstring}
\usepackage{xspace}

\DeclareRobustCommand{\AA}[1]{%
    \IfEqCase{#1}{%
        {1}{\textsf{AA\textsuperscript{\textbf{--}}}}%
        {2}{\textsf{AA\textsuperscript{\textbf{+}}}}%
    }[\PackageError{AA}{Undefined option to AA: #1}{}]%
}%

\DeclareRobustCommand{\UN}{\textsf{UN}\xspace}%

\newcommand{\pie}{\pi^\mathfrak{e}}
\allowdisplaybreaks
\usepackage{fullpage}
\newtheorem{theorem}{Theorem}
\newtheorem*{theorem*}{Theorem}

\newtheorem{lemma}{Lemma}

\newtheorem{assumption}{Assumption}
\newtheorem{proposition}{Proposition}
\newtheorem*{proposition*}{Proposition}
\theoremstyle{definition}

\newtheorem{definition}{Definition}

\newtheorem*{remark*}{Remark}
\date{}
\title{\textbf{From Fair Decision Making To Social Equality}}
\author{Hussein Mouzannar \thanks{American University of Beirut. Email: \texttt{hssein.mzannar@gmail.com}} \and Mesrob I. Ohannessian \thanks{Toyota Technological Institute at Chicago. Email: \texttt{mesrob@ttic.edu}} \and Nathan Srebro \thanks{Toyota Technological Institute at Chicago. Email: \texttt{nati@ttic.edu}}}
\begin{document}
\maketitle
\begin{abstract}
The study of fairness in intelligent decision systems has mostly ignored long-term influence on the underlying population. Yet fairness considerations (e.g. affirmative action) have often the implicit goal of achieving balance among groups within the population. The most basic notion of balance is eventual equality between the qualifications of the groups. How can we incorporate influence dynamics in decision making? How well do dynamics-oblivious fairness policies fare in terms of reaching equality? In this paper, we propose a simple yet revealing model that encompasses (1) a selection process where an institution chooses from multiple groups according to their qualifications so as to maximize an institutional utility and (2) dynamics that govern the evolution of the groups' qualifications according to the imposed policies. We focus on demographic parity as the formalism of affirmative action.

We then give conditions under which an unconstrained policy reaches equality on its own. In this case, surprisingly, imposing demographic parity may break equality. When it doesn't, one would expect the additional constraint to reduce utility, however, we show that utility may in fact increase. In more realistic scenarios, unconstrained policies do not lead to equality. In such cases, we show that although imposing demographic parity may remedy it, there is a danger that groups settle at a worse set of qualifications. As a silver lining, we also identify when the constraint not only leads to equality, but also improves all groups. This gives quantifiable insight into both sides of the mismatch hypothesis. These cases and trade-offs are instrumental in determining when and how imposing demographic parity can be beneficial in selection processes, both for the institution and for society on the long run.
\end{abstract}

\section{Introduction}

As many aspects of human society become increasingly automated, questions of ethical nature migrate into the technological sphere. Admittedly, no amount of formalization can capture all the complexity and subtlety of these issues. Yet, one is compelled to mathematize them and integrate them with familiar frameworks. The result of this process is the design of purportedly ethically-minded technology. An important component of this process, then, is to understand just how well such technology interacts with society and whether its use does indeed result in the intended effect.

Automated decision and policy making, particularly the data-driven case embodied by machine learning, is one such technology that has recently allotted considerable attention to these ethical questions. In particular, guaranteeing non-discriminatory behavior has been at the forefront of research. Various formalizations of non-discrimination have been proposed, relationships between these notions have been studied \cite{dwork2012fairness,hardt2016equality}, and fundamental trade-offs have been identified \cite{kleinberg2016inherent}. Yet the influence of the resulting decision systems on society has only garnered modest treatment \cite{pmlr-v80-liu18c}. One place where this question has been thoroughly considered is in studying the evolution of negative stereotypes in labor markets, in economics \cite{coate1993will, hu2018short}. The dynamics there, however, are very carefully crafted to mimic the labor market. To relate the formalism of the current paper with these prior work, see Section \ref{sec:related}.

This paper starts by proposing a selection process where an institution chooses from multiple groups according to their qualifications so as to maximize an institutional utility.
In this context, non-discrimination often takes the form of constraining the type of selection policies. Demographic parity, for example, enforces groups to have equal selection rates. This is a simple yet rich model that has been well studied \cite{hardt2016equality} \cite{kleinberg2016inherent}, and as such this paper chooses it as the archetype of non-discrimination, referring to it with the more colloquial name of affirmative action (AA). But non-discrimination is rarely a goal in and of itself. Ultimately, one tacitly expects a benefit to society. What kind of expected benefit is implicit in AA? The fact that a well-integrated society does not require non-discrimination hints at a possible answer. This is in particular true if the groups themselves are indistinguishable. If one accepts that such social equality is the intended effect of non-discrimination, then the question becomes: does AA lead to it? To answer this, modeling how the populations change in response to policies is not simply a luxury, but a necessity. This paper starts with some natural tenets, and derives from them a model of dynamics that govern the evolution of the groups' qualifications according to the imposed policies. These dynamics are general enough to model behavior beyond the specific ones in prior work.

The paper then proceeds by first giving conditions under which an unconstrained policy reaches equality on its own. In this case, surprisingly, imposing AA may break equality. When it doesn't, one would expect the additional constraint to reduce utility, however, it can be shown that utility may in fact increase. In practice, AA acts in two possible ways. Either selection rates of the privileged are reduced or those of the underprivileged are increased \cite{fessenden_keller_2013}. This dichotomy is at the heart of the starkly different ways in which non-discrimination manifests itself. In real world scenarios, unconstrained policies do not lead to equality. In such cases, it can be shown that although imposing AA may remedy it, there is a danger that groups settle at a worse set of qualifications. As a silver lining, one can identify exactly when the constraint not only leads to equality, but also improves all groups.

To summarize, the contributions of this paper are as follows:
\begin{itemize}
	\item A simple yet flexible model for both a selection process by an institution and dynamics describing the change in the population due to selection policies
	\item A characterization of the policies resulting from institutional utility maximization without and with imposing affirmative action as non-discrimination 
	\item Conditions under which society equalizes for each type of policy, along with the impact on the long-term institutional utility
\end{itemize}
These cases and trade-offs are instrumental in determining when and how imposing demographic parity can be beneficial in selection processes, both for the institution and for society on the long run.

The rest of the paper is organized as follows. Section \ref{sec:related} covers the most relevant related work. Section \ref{sec:problem} formulates the problem and establishes the models. Sections \ref{sec:unequalityaa} and \ref{sec:aagood} study the contrast between imposing affirmative action or not, in the case when not imposing it naturally leads to equality or not, respectively. Section \ref{sec:conclusion} concludes.

\section{Related Work} \label{sec:related}

On the economics side, a large literature exists examining statistical discrimination, dating back to Arrow \cite{arrow1973theory}. Prior work has attempted to study the effect of affirmative action in the labor market to remove stereotypes \cite{coate1993will, hu2018short,moro2004general,austen2006redistribution,foster1992economic}, where stereotypes are defined as misrepresenting the quality of a worker from a disadvantaged group, when groups are assumed to be ex-ante equal. The models of \cite{coate1993will} and \cite{hu2018short} are the most related to the present work: they both set up a game in a market between workers and employers, where workers have the choice to invest in themselves to become qualified, and emit a test signal that employers threshold as their policy. The dynamics that this game generates on the distribution of qualified individuals per group is very specific. In particular, it is restricted to be unimodal. Both papers look at sufficient conditions for equality and propose interventions to actively reach it. The current paper trades off the concreteness of this line of work for more domain generality, through less specific dynamics. This helps garner more insight about the fundamental interplay between the constraints and changes in society. Moreover, examining dynamics and tracking the populations through time instead of only looking at equilibria as in \cite{moro2004general,foster1992economic}, allows comparing cumulative utility which plays a role in short term considerations. 

A recent work \cite{pmlr-v80-liu18c} raised very similar questions to the ones considered here. The main difference in that setup is that individuals are characterized by a potentially non-binary score and are studied over only a single step of the selection process. The influence of the selection is modeled by a social utility function $\Delta$ that can be interpreted as the expected change in score for an individual due to selection. The authors do not suggest that this change is due to population shifts, though it could be. Without tailoring $\Delta$ specifically to have such an interpretation, it is not possible to track distributions of scores over long time horizons. The dynamics proposed here implements precisely that aspect and enables the study of both institutional utility and social equality that emerge from fairness constraints. Another somewhat related work is that of \cite{hashimoto2018repeated}, which models the retention rate of users of a particular service as a function of that service's quality. There, this is thought of as the error of a learning algorithm and the authors propose a robust learning approach to avoid amplifying the difference in retention rate by group.

The present formulation also bears some similarity to reinforcement learning due to the dynamics being essentially Markovian. As such, it can be thought of as an optimal control problem. The goal here however is to analyze policies that are unaware of the dynamics, and do not make an explicit effort to learn it. Also, the focus is not simply on utility, but the secondary objective of social equality too. Previous work \cite{jabbari2017fairness} has explored fairness with learning, guaranteeing that actions are preferentially selected only if they effectively generate more reward. However, there are many key differences, such as the state and action space in the current model being continuous as opposed to the discrete versions treated in \cite{jabbari2017fairness} and having deterministic transitions. Work in online decision making \cite{joseph2016fairness} has also explored fairness constraints, but in a bandit setting that lacks the dynamical aspect.

\section{Model} \label{sec:problem}

\subsection{Society: groups, qualification profiles}
Throughout this paper, the total population is divided into two groups indexed by the variable $G$: group $G=A$ which constitutes the $g_A$ fraction and group $G=B$ the $g_B= 1- g_A$ fraction. The limitation to only two groups simplifies the exposition yet maintains the essence of the problem.

One can conceptually think of each individual in either group as having some attribute $\theta \in \Theta$ that bears information about qualification. For example in a college admission scenario the attribute $\theta$ can be: $[\text{GPA, SAT, Letters of recommendation}]$. This attribute is thought to provide a complete description in this context. In what follows, $\theta$ is implicitly mapped through an estimator $F: \theta \to \lbrace0,1\rbrace$ to a crisp evaluation of qualification: say $v=1$ if qualified and $v=0$ otherwise. The assumption is that this binary classification is sufficient to characterize the utility that the individual would bring. This is compatible with the work on statistical discrimination \cite{coate1993will,hu2018short}. An alternative way to reach this model is to first consider a more general multiple class characterization and then show that a threshold policy is optimal with respect to utility and societal improvement, as was done in \cite{pmlr-v80-liu18c}. This again indicates that a binary characterization, i.e. above and below the threshold, is often sufficient. For the example of admissions, one can think of $v=1$ as indicating an individual likely to be successful in college. 

We denote the distribution of evaluations $V=v$ in group $G=j$ by $\pi(\cdot|G=j)$. This probability distribution is the \emph{qualification profile} of group $G=j$. Instead of tracking individuals, it is more natural to study changes in the population in terms of the qualification profiles of groups. For example, in college admissions it is not the same students that would change their qualifications in response to admission policies, but rather those of the next application cycle.

\subsection{Institution: policy, utility, selection rates}
The population undergoes a selection process as illustrated in figure \ref{fig:pipeline}. An \emph{institution}, referred to also by \emph{policy maker}, designs a \emph{policy} $\tau(V=v;G=j)$: $\{0,1\} \times \{A,B\} \rightarrow [0,1]$ that maps each individual to a probability of selection, possibly depending on the group identity.
The institution  places a \emph{utility} on each possible evaluation $v$, through a map $u:\{0,1\}\to\mathbb{R}$, $v\mapsto u(v)$. Since $v=1$ and $v=0$ are assumed to be beneficial and detrimental respectively, and to avoid trivial policies where none/all are selected in what follows, assume $u(0)\leq 0 \leq u(1)$.

 The (average) \emph{institutional utility} of a policy can then be defined as follows:
\begin{equation} \label{eq:utility}
U(\tau) = \sum_{j \in \lbrace A,B \rbrace} g_j \sum_{v\in \lbrace 0,1 \rbrace } u(v) \cdot \tau(V=v;G=j) \cdot \pi(v|G=j).
\end{equation}

Any given policy also defines (per group) \emph{selection rates}: $\beta(G=j) = \sum_{v\in \lbrace 0,1 \rbrace }  \tau(V=v;G=j) \cdot \pi(V=v|G=j)$. For college admission, this corresponds to acceptance rates per group. Note that each term in the sum is the rate corresponding to a particular evaluation. Denote these  by $\beta(V=v;G=j) = \tau(V=v;G=j) \cdot \pi(V=v|G=j)$. In what follows, for notational convenience, the explicit '$V=$' and `$G=$' are dropped from the expressions of the qualification profiles, policies, and selection rates, since the context of the argument is clear.

\begin{figure}    
\centering
  \includegraphics[trim={0 1in 0 1in},clip,scale=0.5]{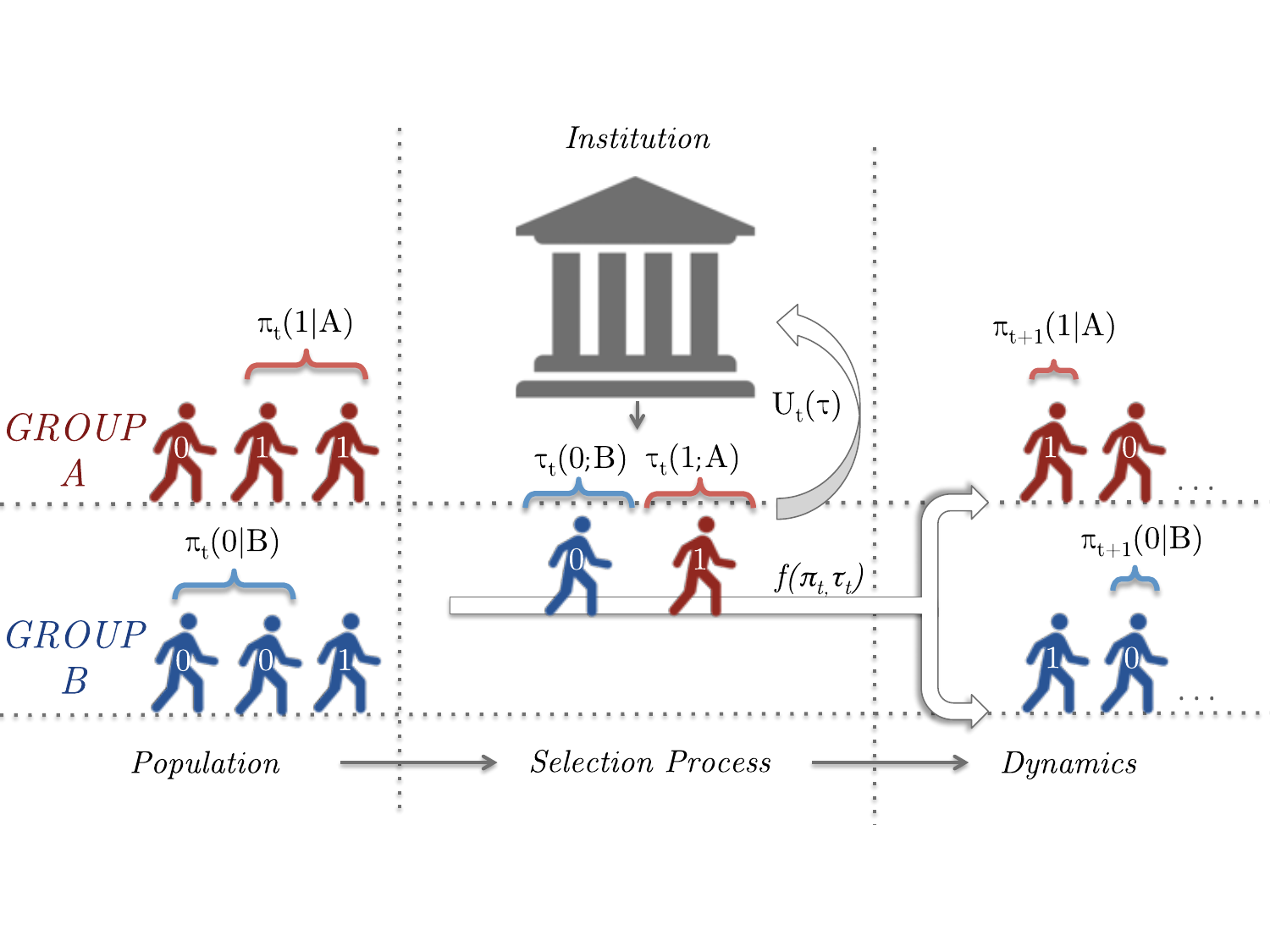}
    \caption{Selection process and influence dynamics -- Two subgroups of the population have each publicly known qualification profiles $\pi$. An institution performs a selection by maximizing utility $U$, possibly subject to a non-discrimination constraint. The selection policies $\tau$ influence the profiles at the next round, through dynamics $f$.}
    \label{fig:pipeline}
\end{figure}

\subsection{Influence Dynamics}
The execution of a selection policy can be thought of as demarcating time $t$. The main premise of this paper is that the selection process affects the population, namely by changing the qualification profiles of either group at time $t+1$. This change can be interpreted in different ways. It could be due to genuine change within society: the institutions' policies can then be thought of as having a secondary effect as social incentive or deterrent. But change could also be due to self-selection in the pool of individuals available to the institution on the next time step: policies then have the effect of filtering through time, possibly without society itself changing. Although the model is not specific to an interpretation, when the perspective moves to social equality, the tacit assumption is that dynamics are due to genuine change.

Dynamics are in general very complicated. They may vary in time, be stochastic, and depend on many parameters of the problem. Some of these aspects, however, are more relevant than others for capturing the fundamental effects. In what follows, dynamics are taken to be time-invariant and deterministic. Time-invariance is a convenient assumption, which can be supported by the slow change in the way in which society responds to stimuli. Determinism does not mean that there is no stochasticity within society, but that any such change is summarized through the evolution of the qualifications.
As such it is a Markovian assumption, with the state consisting of the qualification profiles.

Finally, what parameters of the problem should govern change and in what manner should they do so? There are three tenets that lead to this paper's proposal:
\begin{itemize}
\item[(a)] Within each group, there is an (a priori) distribution of \emph{potential} qualifications, identical to the qualification profiles at the previous time step.
\item[(b)] Upon observing the previous time step's policy, individuals respond by instantiating with some probability as (a posteriori) qualified or unqualified.
\item[(c)] The response probabilities, which summarize how policies influence individuals, depend only on the selection rates within the group of the individual and their potential qualification.
\end{itemize}

Tenet (a) describes society's inertia in the absence of influencing policies. Tenet (b) could be thought of as a result of adequate or lacking individual effort in response to policies. In tenet (c), the fact that potentially qualified and potentially unqualified individuals may be influenced differently is necessary for the history of qualifications to have relevance. The main restriction thus is that the only influence of policies is through selection rates within a group. Lack of influence across groups is an idealization of such influence arguably being weaker than within a group. As for selection rates being the key variable rather than the selection policy itself, one can motivate it by considering extremes: indeed, if no qualified (or unqualified) individuals exist then selection policy for qualified (or unqualified) individuals is never observed and cannot have influence.

With these tenets, dynamics induced by the influence of policies can be defined as follows.
\begin{assumption}[Dynamics]
	For a given group $j$ , let $\pi_t(1|j)=:\pi_t(1)$ (for clarity, the group index is dropped here and later) denote the qualification profile of group $j$ for $v=1$ at time $t$ and let the policies $\tau_t$ at that time step induce the selection rates $\beta_t$. Then the qualification profiles at time $t+1$ are given by 
      \begin{equation}\label{eq:dynamic}
	      \pi_{t+1}(1) = \pi_t(1)~f_1\Big( \beta_t(0),\beta_t(1)\Big) \ +\  \pi_t(0)~f_0\Big(\beta_t(0),\beta_t(1)\Big)
      \end{equation}
    where $f_0$ and $f_1$ are two arbitrary continuously differentiable functions from $[0,1] \times [0,1] \rightarrow [0,1]$. The pair $(f_0, f_1)$ is referred to as the \emph{dynamics}.
\end{assumption}

These dynamics make the tenets concrete. For each group $G=j$, $\pi_t(\cdot|j)$ describes the potential qualification profile. The function $f_1$ represents the \emph{retention at the top}: the rate of retention of the sub-population with potential $v=1$ due to the current policy. The function $f_0$ represents \emph{change for the better}: the fraction of the potential $v=0$ sub-population that progresses to have evaluation $1$. Equivalently, one can pair these with their counterparts to describe full conditional distributions. The pair $1-f_1$ (\emph{change for the worse}) and $f_1$ are the conditional distribution describing the response of an individual with potential $v=1$ into actual evaluations $0$ or $1$ respectively, while the pair $1-f_0$ (\emph{retention at the bottom}) and $f_0$ are the same respectively for an individual with potential $v=0$.

It is worth making a couple of remarks about these dynamics. First, the two groups' responses are identical since the choice of $f_1$ and $f_0$ does not depend on the group.
To justify this, note that in fairness considerations such as demographic parity one makes the inherent assumption that the groups are ex-ante equal in all respect except for their qualification profiles.
Second, the groups do not interact through the functional form of the dynamics. Differences between groups and any potential coupling between groups can only happen through the different and interacting selection rates induced by the policies.

With the model of the selection process and dynamics in hand, one can now start to formulate the main question: when and how can  policies lead to social equality? Perhaps the simplest notion of equality within the population can be defined as the groups becoming indistinguishable in their qualification profiles. This is formalized as follows.
\begin{definition}[Social Equality]
	Under given dynamics $(f_0, f_1)$, a policy is said to be \emph{equalizing} if for all starting $\pi_0(1|A)$ and $\pi_0(1|B)$
    \[
    	\lim_{t\to \infty} \ \left|\pi_t(1|A)-\pi_t(1|B)\right|= 0
    \]
    When this happens, say that the policy/population \emph{reaches equality}, which is understood to be in an asymptotic sense.
\end{definition}

Our goal in the following sections is to characterize the behavior of society in terms of utility and social equality under different policies which we will next describe.
\subsection{Utility Maximization and Non-discrimination}
The primary tendency of the institution is to maximize its average utility. Rarely, if ever, are institutions aware of the underlying dynamics. Thus \emph{dynamics-oblivious} policies are the appropriate ones, where the utility is maximized over a single time step. Next, two approaches taken in practice are modeled: unconstrained maximization without any non-discrimination considerations and affirmative action via demographic parity.

Left unconstrained, institutional utility maximization amounts to solving the following linear program (LP) at every time step $t$:
\begin{equation}\label{eq:1-step-un}
\max_{\tau} \ U_t(\tau),
\end{equation}
where $U_t$ is as in \eqref{eq:utility}, with all quantities defined at time step $t$. 

It is  straightforward to see that the optimal policies are $\tau_t(1;\cdot)=1$ and $\tau_t(0;\cdot)=0$ for all time $t$; this is the unconstrained (\UN) policy.

Note that the policy only depends on the distribution at time $t$. Any dynamics for group $j$ is thus simply of the form:
\begin{equation*}
\pi_{t+1}(1|j) = f(\pi_{t}(1|j)),
\end{equation*}
where $f$ is defined in terms of $(f_0,f_1)$ according to Equation \eqref{eq:dynamic} as:
\begin{equation}
    	f(\pi) := \pi f_1(0,\pi)+(1-\pi) f_0(0,\pi).
\label{eq:fundynamic}
\end{equation}

If the evaluation $v$ is perfect, then the unconstrained policy is trivially \emph{identity blind}, thus one could argue that it is fair in the sense of equalized opportunity \cite{hardt2016equality}. However, as distributions of qualified individuals differ per group, the policy will have different selection rates per group. \emph{Affirmative action} (AA), at an intuitive level, attempts to bridge such inequalities between different groups. More precisely, AA forces equal aggregate selection rates between groups \cite{coate1993will,hu2018short}, and it is more formally known as demographic parity \cite{barocas-hardt-narayanan,pmlr-v80-liu18c}. This paper adheres to this as the archetype of non-discrimination.
\begin{definition}[Affirmative Action]
	The \emph{affirmative action} constraint forces the policy to select at an equal rate between both groups:  
	\begin{equation} \label{eq:dem-parity-constraint}
		\beta(A) = \beta(B)~.
	\end{equation}
\end{definition}

Thus in this case, the policy maker solves \eqref{eq:1-step-un} with the additional constraint \eqref{eq:dem-parity-constraint}. Since the constraint is linear in $\tau$, the optimal policy is also found via a LP;
the following characterizes the solution of the LP. The proof can be found in Appendix \ref{apx:proofs}.

\begin{proposition}\label{prop:aalp}
Assume that at time $t$ group $j$ is \emph{advantaged}, defined as $\pi_t(1|j)\geq \pi_t(1|\neg j)$. Then the optimal policy is one of two cases, depending only on $g_j$, $u(0)$, and $u(1)$:

	\begin{itemize}
		\item[\AA1]  \textrm{If} \ $g_j u(1) + (1-g_j) u(0)\leq 0$, then  
        \\$\tau_t(1;j)= \frac{\pi_t(1|\neg j)}{\pi_t(1|j)}$, $\tau_t(0;j)=0$ (under-acceptance),
        \\$\tau_t(1;\neg j)=1$, $\qquad\ \tau_t(0;\neg j)=0$~. 
		\item[\AA2]  \textrm{If} \ $g_j u(1) + (1-g_j) u(0)\geq 0$, then
        \\$\tau_t(1;j)= 1$, $\ \ \tau_t(0;j)=0,$
        \\$\tau_t(1;\neg j)=1$, $\tau_t(0;\neg j)=  \frac{\pi_t(1|j) -\pi_t(1|\neg j)}{1 - \pi_t(1|\neg j)}$ (over-acceptance).
	\end{itemize}
\end{proposition}

There are two possible cases through which affirmative action impacts the policy, denoted by \AA1 and \AA2. These represent two drastically different approaches to fairness. \AA1 (\emph{under-acceptance}) accepts fewer qualified individuals from the advantaged group so as to equalize the selection rates for qualified individuals between both groups. On the other hand, \AA2 (\emph{over-acceptance}) accepts unqualified individuals from the disadvantaged group in order to equalize the aggregate selection rates. One could think of \AA1 as increasing the standard for the advantaged group and as such reducing total selection rates. As a peek at the intuition to develop, this could possibly reduce the motivation for individuals to become qualified. As for \AA2, one could think of it as reducing the standard for the disadvantaged group. It is not as evident whether this could potentially increase or decrease the motivation to become qualified. One one hand, it does present the possibility of leading to a more qualified overall society by increasing the total selection rates and training more unqualified people. On the other hand, selecting unqualified corresponds to the condition of the so-called \emph{mismatch hypothesis} \cite{sander2004systemic}, which claims that this could actually reduce the number of qualified individuals over time. This mismatch effect has been amply debated in discussions on public policy with arguments on both sides \cite{ho2004affirmative,arcidiacono2011does,sander2004systemic}. One perspective of the current investigation is to understand how the dynamics tip the scales in this debate.

Note that under AA, the dynamics for one group depends on the other group: in Proposition \ref{prop:aalp} there always exists one $j$ for which the policy depends on the qualification profiles of $\neg j$. The dynamics therefore have to be tracked jointly, unlike in \UN.

Write
\begin{align*}
&\pi_{t+1}(1|A) = f_A(\pi_{t}(1|A),\pi_{t}(1|B)) \\
&\pi_{t+1}(1|B) = f_B(\pi_{t}(1|A),\pi_{t}(1|B))~,
\end{align*}
where $f_A$ and $f_B$ are two components of a joint dynamics function $f:[0,1]\times[0,1]\to [0,1]\times [0,1]$. These components can be written in terms of $f_0$ and $f_1$ according to Equation \eqref{eq:dynamic} and Proposition \ref{prop:aalp}, but have a complicated expression that accounts for all four cases (who is advantaged and which AA case it is). Nevertheless, one could say that given dynamics of this form, equality will be reached if as $t\to \infty$ $\left| f^t_A(\pi_{0}(1|B),\pi_{0}(1|A))-f^t_B(\pi_{0}(1|B),\pi_{0}(1|A))\right| \to 0 $, for all $\pi_0(1|A)$ and $\pi_0(1|B)$.

\subsection{Dynamics in Continuous Time}
While discrete time (DT) describes successive selection-response steps naturally, the arbitrary nature of the dynamics can lead to sudden jumps and oscillatory behavior. These do not quite correspond to how populations evolve in the real world, where one expects change to happen gradually and slowly through time. Figure \ref{fig:bachdegreebysex} shows the evolution of the percentage of bachelor degree holders by sex in the United States of America. In 1967, 13\% of men 25 years and older held a bachelor's degree or higher while 8\% of women did \cite{ryan_bauman_2016}; this gap has decreased through time as degree attainment increased for both sexes. In 2015, 33\% of women held a college degree compared to 32\% of men, thus the proportion of degree holders equalized between both sexes in 47 years time. One of the reported reasons behind this increase has been attributed to Title IX \cite{titleix}, a civil rights law passed in 1972 aiming to eliminate discrimination in educational programs and the removal of quotas against disadvantaged groups. 

\begin{figure}
\centering
\begin{tikzpicture}[scale=1.2]
\begin{axis}[
    /pgf/number format/.cd,
        use comma,
        1000 sep={},
	axis x line=bottom,
	axis y line=left,
xmin=1967, xmax=2015,
ymin=5, ymax=40,
xtick={1967,1973,1979,1985,1991,1997,2003,2009,2015},
    ytick={
5,10,15,20,25,30,35,40},
ymajorgrids=true,
grid style=dashed
] 
\addplot[
    color=black,
    mark=square,
    ]
    coordinates {
(1967,	13)
(1973,	15.5)
(1979,	19)
(1985,	22)
(1991,	23.3)
(1997,	25)
(2003,	27.5)
(2009,	30)
(2015,	32.3 )   };
 \addlegendentry{men}  
\addplot[
    color=gray,
    mark=*,
    ]
    coordinates {
(1967,	8)
(1973,	9.5)
(1979,	12)
(1985,	15)
(1991,	17.5)
(1997,	21)
(2003,	24)
(2009,	29)
(2015,	32.6 )     };
   \addlegendentry{women}
\end{axis}
\end{tikzpicture}
\caption{Percentage of the US population 25 years and older holding a bachelor degree by sex from 1967 till 2015, adapted from \cite{ryan_bauman_2016}.}
\label{fig:bachdegreebysex}
\end{figure}
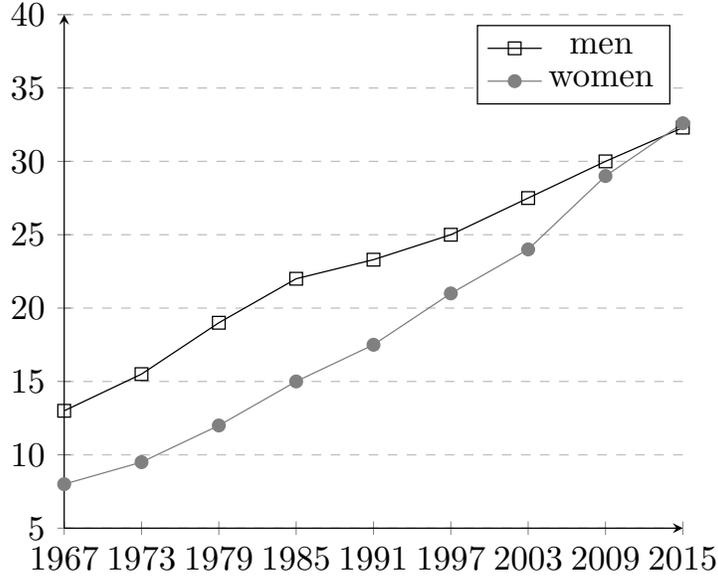

A change that is smooth and gradual can in fact be modeled by transforming the discrete time dynamics into continuous time (CT). The transformation is standard, starting from dynamics \eqref{eq:dynamic} and dropping group indices, we define the continuous time analog of any dynamic as follow.

\begin{definition}[Continuous Time Dynamics]
Given dynamics \eqref{eq:dynamic} $(f_0,f_1)$ we define the continuous time analog as:
\begin{equation}
\frac{d \pi_t}{dt} =  \pi_t  \big(f_1\left(\beta_t\left(0\right),\beta_t\left(1\right)\right)-1\big) \ + \ \left(1-\pi_t\right)  f_0\left(\beta_t\left(0\right),\beta_t\left(1\right)\right) \label{eq:CTdynamic}
\end{equation}
\end{definition}

The CT dynamics thus derived \eqref{eq:CTdynamic} can also be motivated in two additional ways. As shown through technical lemmas in appendix \ref{apx:proofs}, it maintains the properties of the DT dynamics regarding conditions leading to social equality. Moreover, it possesses additional properties that greatly facilitate the analysis.
Therefore, in what follows, the analysis of the selection process on the groups' distributions due to the different policies is performed under CT dynamics. The treatment could have initially started with CT dynamics, however the setup is more intuitive in DT and motivates the transition better.

\section{Affirmative Action under Natural Equality} \label{sec:unequalityaa}

This section is concerned with the following scenario: assume that using the \emph{unconstrained} policy equality is reached, what can happen if one had instead followed an affirmative action policy? Even though affirmative action may seem unnecessary, as without requiring further intervention equality is reached, but can there be any additional benefit to applying fairness constraints and is one guaranteed that in this setting equality will also be reached? This is relevant since the policy maker may not be aware of operating in this regime.

For the unconstrained policy to reach social equality, it is necessary that the dynamic under it's effect has a unique globally attracting equilibrium point, this implies that the iterates of the dynamics converge to a constant function on $[0,1]$. If this convergence is uniform, then there exists an iterate $T$ of $f$ such that $f^T$ is contractive, we pose a stronger assumption such that  $f$ is contractive. 
\begin{assumption}
\label{as:fcontractive}
Recall that the dynamics under \UN can be written as:
\[
	f(\pi) := \pi f_1(0,\pi)+(1-\pi) f_0(0,\pi).
\]
We assume that $f$ is $L_{\UN}$-Lipschitz with $L_{\UN}<1$, meaning that $\forall \pi,\pi' \in [0,1]$:
\begin{equation}
    |f(\pi)-f(\pi')| \leq L_{UN} | \pi - \pi'|
\end{equation}
\end{assumption}
By Banach's fixed-point theorem \cite{latif2014banach},  we have that the \UN policy reaches social equality; refer to Appendix \ref{apx:proofs} for a detailed discussion of the implications of assumption \ref{as:fcontractive}. 
 The precise questions are then: are the conditions for equality with affirmative action also met? And how does the cumulative utility obtained under AA policies compare to that under the \UN policy?

Under AA, the analysis of equality gets slightly more complicated. This is mainly due to the fact that in Proposition \ref{prop:aalp} there is a specific advantaged group. If from one time step to the next the advantage changes, the identities of the groups swap. This in itself is not an issue, unless the swap also causes a switch between the cases \AA1 and \AA2. This last situation will not be considered for two reasons: first, because it is too unwieldy, and second, because it does not arise if the dynamics are slow enough, for example by taking the point of view of continuous time.

The following characterizes the behavior of the institution adopting an \AA1 (under-acceptance) policy under the stated conditions.
\begin{theorem}\label{th:UNimpliesAA1}
If equality is reached with an unconstrained (\UN) policy by way of assumption \ref{as:fcontractive}, then it is necessarily reached by an \AA1 policy implemented over all time steps, however with no more, and possibly less, utility at each step.
\end{theorem}
The proof of the theorem and further discussion into the analysis can be found in appendix \ref{apx:proofs}.

In light of this, the policy maker has no gain in enforcing affirmative action in this case. Under acceptance will maintain the state of social equality but cause a loss of utility

Consider next how \AA2 acts on the population in this setting. The conditions for equality with the \AA2 policy cannot be directly implied from the \UN equality; there exists dynamics under which the \UN policy reaches social equality but the \AA2 policy does not.

However, under \AA2, the advantaged group's (A) distribution trajectory will reach the equilibrium point $\pie$ as the group's policy is identical to the \UN policy in CT. On the other hand, group B's distribution evolution through time is unclear. The dynamics governing it is, with $\Delta(t)=  \pi_t(1|A,\AA2) -  \pi_t(1|B,\AA2)$:
\begin{align}
\frac{d\pi_t(1|B,\AA2)}{dt} &= \pi_t(1|B,\AA2)(f_1(\Delta(t),\pi_t(1|B,\AA2) -\Delta(t)) \nonumber\\ &+ (1-\pi_t(1|B,\AA2)) (f_0(\Delta(t),\pi_t(1|B,\AA2) -\Delta(t))- \pi_t(1|B,\AA2)\nonumber
\end{align}
Since $\pi_t(1|A,\AA2)$ converges to $\pie$ as $t\to \infty$, then $\Delta(t)$ will only become a function of $\pi_t(1|B,\AA2)$ and $\pie$. Therefore, the dynamics will approximately become an ODE with one variable, and oscillatory behavior cannot exist in first order one dimensional ODEs \cite{strogatz2018nonlinear}, then $\pi_t(1|B,\AA2)$ will converge to some fixed point of the dynamics necessarily below $\pie$. Looking closely, if $\pie$ was initially below $\pi_0(1|B,\AA2)$, then the trajectory of $\pi_t(1|A,\AA2)$ will force group B to converge toward the equilibrium point. On the other hand, if $\pie\geq\pi_0(1|B,\AA2)$ then nothing can be said, thus unfortunately only worse off equilibria are guaranteed.

Now let's assume the conditions for reaching equality with the \AA2 policy are met to see what could be the consequence.
\begin{assumption}
\label{as:aa+contractive}
If the form of AA coincides with \AA2 then $f_A$ and $f_B$ for all $\pi,\pi' \in [0,1]$ can be written as:
\begin{align*}
    &f_A(\pi,\pi') = \pi f_1([\pi'-\pi]_+,\pi) + (1-\pi) f_0([\pi'-\pi]_+,\pi) \\
    &f_B(\pi,\pi') = \pi' f_1([\pi-\pi']_+,\pi') + (1-\pi') f_0([\pi-\pi']_+,\pi')
\end{align*}
Assume then that there exists $L_{\AA2}\in[0,1)$ such that for all $\pi,\pi'$ in $[0,1]$ one has:
\begin{equation*}
|f_A(\pi,\pi')-f_B(\pi,\pi')| \leq L_{\AA2} |\pi - \pi'|
\end{equation*}
\end{assumption}

Similarly to the discussion in the \AA1 section, by comparing the cumulative utility with \UN one surprisingly finds a lifeline for enforcing affirmative action.
\begin{theorem}
\label{th:aa2moreutility}
Let $\alpha = (1-g_A)u(1)/((1-g_A)u(1) +|u(0)|)$. Given assumptions \ref{as:fcontractive} and \ref{as:aa+contractive} and implementing an \AA2 policy over all time steps, then if $L_{\UN}$ and $L_{\AA2}$ satisfy the following:
\begin{equation*}
L_{\UN}\geq 1 - \alpha
\end{equation*}
and
\begin{equation*}
 L_{\AA2} \leq 1 + ( L_{\UN} -1)/\alpha,
\end{equation*}
\AA2 provides more utility in CT over an infinite time horizon.
\end{theorem}

It is straightforward to see that $L_\UN \leq 1+(L_\UN-1)/\alpha$. A weaker version of this theorem can thus be informally stated as follows: if under the unconstrained policy equality is reached slowly enough yet under \AA2 equality is reached faster, then \AA2 results in a utility gain. This implies that, rather than trading off, speed of convergence and utility go hand in hand.

\section{Affirmative Action under Disparate Equilibria} \label{sec:aagood}

\subsection{Multiple Equilibria}

It is more realistic to expect that myopically maximizing utility is unlikely to equalize the qualifications of groups. That is, one expects following an unconstrained policy to not lead to equality. To model this scenario, the dynamics should not have a unique attracting equilibrium point. In general one could have arbitrarily many, even infinitely many, fixed points for the dynamics. In this section this is simplified to assuming that the dynamics under \UN has a finite number of attracting equilibrium points, each with its own basin of attraction. Thus the \UN policy would not equalize the groups unless they are initially close enough to fall in the same basin. The notion of such dynamics can be formalized as follows.

\begin{definition}[$k$-Equilibrium Dynamics]\label{def:k-eq-dynamic}
    The function $f$ obtained from $(f_0, f_1)$ under the \UN policy is called a $k$-equilibrium (DT or CT) dynamics, if it is continuously differentiable and has $k$ fixed points $\pie_{1}\leq \cdots \leq \pie_{k}$ that are locally attracting (DT or CT) equilibrium points. Namely, there exist $k-1$ delimiting fixed points $\delta_1\in(\pie_{1},\pie_{2}), \cdots, \delta_{k-1} \in(\pie_{k-1},\pie_{k})$, such that for all $i\in[k]$ and all $\pi_0 \in (\delta_{i-1},\delta_{i})$, using the convention $\delta_0=0$, $\delta_k=1$, and semi-closed basins $[0, \delta_1)$ and $(\delta_k,1]$:
	\[
	    \pi_t \to \pie_{i} \quad \textrm{as}\quad  t \to \infty,
    \]
    where $\pi_t=f^t(\pi_0)$ (for DT) or $\pi_t$ is the solution at $t$ of $\frac{d\pi}{dt} = f(\pi)-\pi$ initialized at $\pi_0$ (for CT).
\end{definition}

\begin{remark*} Note that it follows that the $\delta_i$, $i\in[k-1]$ are unstable equilibria, under both DT and CT dynamics, by virtue of being fixed points. The extremes $\pi=0$ or $\pi=1$ could also be fixed points. If one or the other is an attracting equilibrium, then the notation implies $\pie_1=0$ or $\pie_k$=1 respectively. If one or the other or both is an unstable equilibrium, then $\delta_0$ or $\delta_k$ or both join $\{\delta_i~:~i=1,\cdots,k-1\}$ as unstable equilibria.
\end{remark*}

Figure \ref{fig:keqdynamic} illustrates a $3$-equilibrium CT dynamics under an \UN dynamics with the direction of the gradient illustrated at each point. Observe the $9$ possible joint equilibria that result by following the \UN policy, depending on where each group is initialized. See Appendix \ref{apx:figure-details} for the specific dynamics $(f_0,f_1)$ for this illustration.

\begin{figure}
\centering
  \includegraphics[scale=0.5]{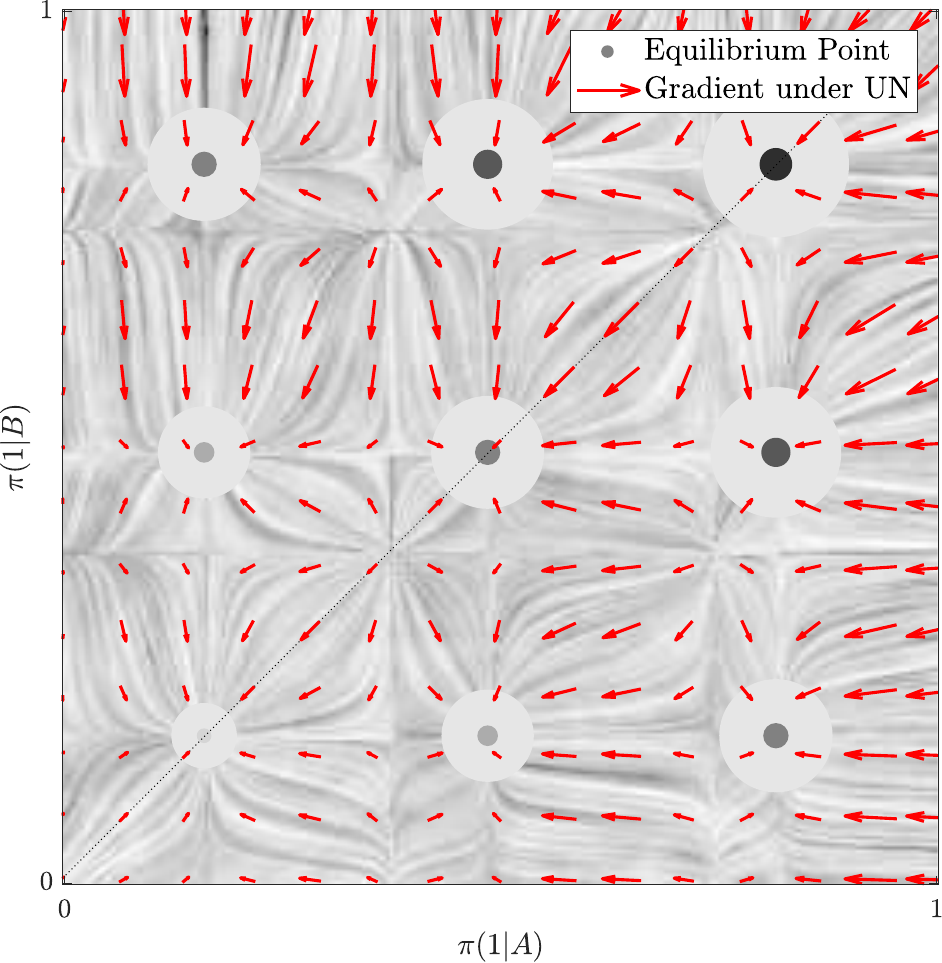}
    \caption{3-equilibrium mapping under \UN policy}
\label{fig:keqdynamic}
\end{figure}

\begin{assumption}
\label{as:keqUN}
    Assume that the function $f$ under the UN policy is a k-equilibrium dynamics and that there exists a neighborhood of radius $r^i>0$ around every equilibrium point $\pie_{i}$ such that:
    \begin{equation}\label{cond:k-eqDT2}
\left|\frac{df}{d\pi}\right|\leq L_i 
\end{equation}
for some $L_i\in[0,1)$  for all $\pi \in [ \pie_{i}-r^i,\pie_{i}+r^i]$.
\end{assumption}

Assumption \ref{as:keqUN} replaces assumption \ref{as:fcontractive} of the previous section to allow for multiple equilibria. Equation \eqref{cond:k-eqDT2} for being locally Lipschitz around each equilibrium point controls the rate of convergence in both DT and CT dynamics, just like was done for the single equilibrium case.

So far we have left $f_0$ and $f_1$ entirely arbitrary which is good for generality but may lead to esoteric behavior. It may be more reasonable to restrict them in certain natural ways. One further assumption we will make in this section is that change is harder than retention.
\begin{assumption}[Status quo bias] \label{status-quo-bias}
	For all $x,y\in[0,1]$: 
	\[
		f_1(x,y)\geq f_0(x,y).
	\]
\end{assumption}

\subsection{Eliminating Disparate Equilibria}
The following shows that similarly to the single equilibrium \UN, the \AA1 policy is equalizing even under a $k$-equilibrium. 

\begin{theorem}
\label{th:AA1equalitywith-keq}
Assume status quo bias, i.e. Assumption \ref{status-quo-bias} and a k-equilibrium dynamics under UN, i.e. Assumption \ref{as:keqUN}. Let $j$ be the initially advantaged group. If the disadvantaged group starts at $\pi_0(1|\neg j)\neq \delta_{i}$ for any $i\in \{1,\cdots,k-1\}$, following \AA1 reaches equality in both DT and CT.
\end{theorem}

The above theorem states that following \AA1 is a guaranteed way to reach equality, however equality comes at a price. Even more generally than Theorem \ref{th:AA1equalitywith-keq}, whenever \AA1 equalizes a $k$-equilibrium CT dynamics , it always leads to worse long-term utility than under \UN by leading to a population with lower qualification. On the other hand when following \AA2, equality is always beneficial. The caveat is that, as in the single equilibrium case, under the conditions of a $k$-equilibrium dynamics, one cannot deduce conditions for \AA2 that lead to equality. Assuming that these conditions do hold, however, guarantee a utility no worse than under \UN by leading to a population with higher qualification. The following Theorem characterizes this.

\begin{theorem}\label{th:keqAA1-2}
Under the conditions of Theorem \ref{th:AA1equalitywith-keq}, if the policy is \AA1, then the equalized population generates long-term utility no higher (and possibly lower) than the limiting population under \UN. If the policy is \AA2 and it leads to social equality, then the equalized population generates long-term utility no lower (and possibly higher) than the limiting population under \UN.
\end{theorem}

\begin{figure}[ht]
\centering
	\includegraphics[scale=0.6]{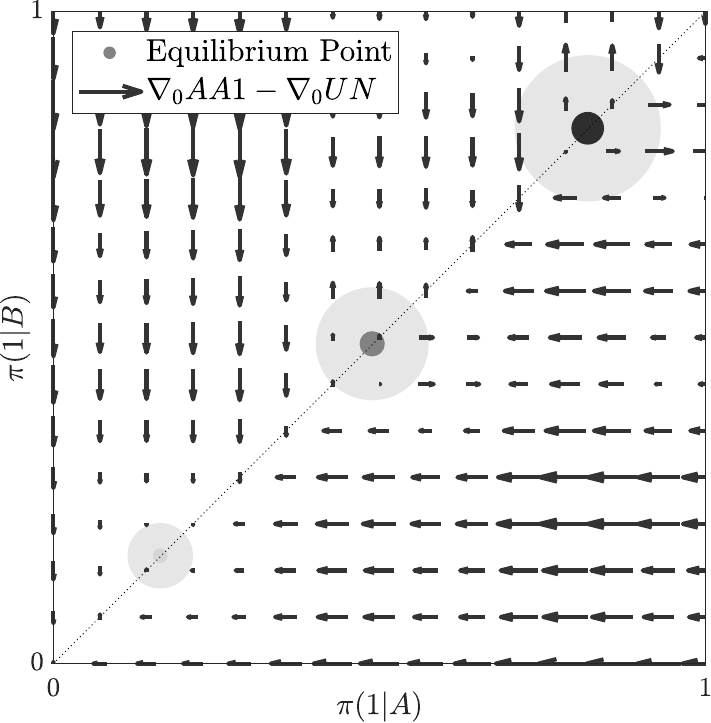}
	\includegraphics[scale=0.6]{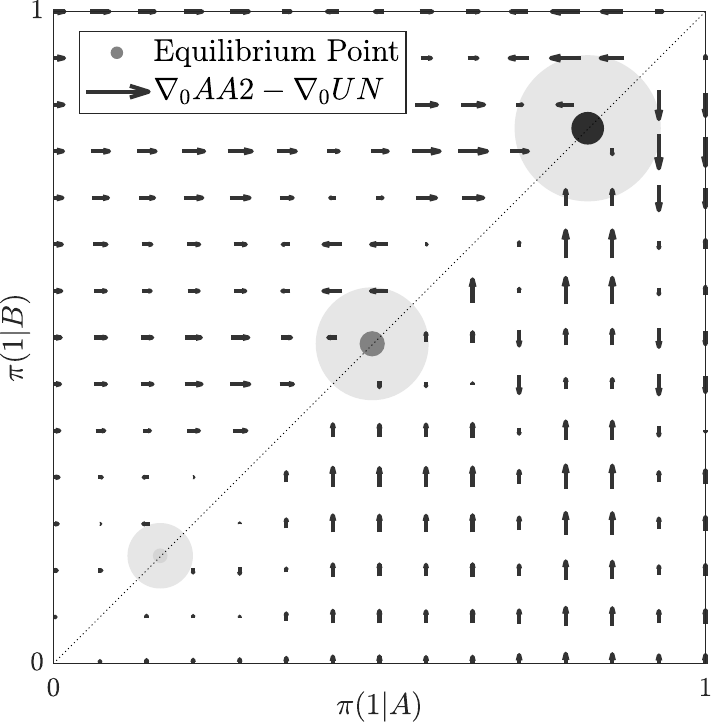}
    \caption{Illustration of Theorem \ref{th:keqAA1-2} for the \AA1 and \AA2 policies respectively under the same dynamics as that of Figure \ref{fig:keqdynamic}}
    \label{fig:aa12dyn}
\end{figure} 

Figure \ref{fig:aa12dyn} shows the difference between the gradient of $\pi_t$ under either \AA1 and \AA2 and \UN, for the dynamics which resulted in Figure \ref{fig:keqdynamic} in a $3$-equilibrium \UN. This can be shown to satisfy the conditions to make the \AA1 and \AA2 policies equalizing. Note how the gradient difference for \AA1 points generally downward and leftward toward states with lower utility  as opposed to that of \AA2 which generally points upward and rightward toward equilibrium points with higher utility. These illustrate qualitatively the kind of behavior characterized in Theorem \ref{th:keqAA1-2}.

\section{Conclusions} \label{sec:conclusion}

Imposing fairness considerations on decision making systems requires understanding the influence that they will have on the population at hand. Without knowledge of this influence, seemingly fair constraints might in fact exacerbate existing differences between groups of the population.

This paper proposed a simple but expressive model of a selection process to concretely consider these concerns. Namely, affirmative action was studied in terms of its ability to equalize the qualifications of different groups. Imposing one case of affirmative action (under-acceptance of qualified individuals) was shown to guarantee equality at the cost of worse institutional utility and possibly decreasing the population's overall qualification level. In another case (over-acceptance of unqualified individuals), however, affirmative action was shown to lead to a policy with different characteristics: equality cannot be directly guaranteed to hold but when it does, it results, in equilibria where the population becomes more qualified. Interestingly, this last case corresponds to the debate on the mismatch hypothesis. The analysis here quantifies how both sides of the debate can be correct, depending on whether society equalizes or not. This invites devoting attention to the conditions under which over-acceptance tips society toward equality.

This paper is certainly not the last word on the role of dynamics in non-discrimination. The hope is to spur many new lines of investigation. The proposed model and analysis could serve as a framework to evaluate other non-discrimination constraints. More ambitiously, it could be used to motivate and derive new constraints with explicit awareness of their long term impact. 



\bibliographystyle{alpha}
\bibliography{ref}
\newpage
\appendix
\section{Deferred proofs and discussion}\label{apx:proofs}
\subsection{Section \ref{sec:problem}}
\subsubsection{Proposition \ref{prop:aalp}}
\noindent \paragraph{Proposition \ref{prop:aalp}}
\textit{Assume that at time $t$ group $j$ is \emph{advantaged}, defined as $\pi_t(1|j)\geq \pi_t(1|\neg j)$. Then the optimal policy is one of two cases, depending only on $g_j$, $u(0)$, and $u(1)$:
	\begin{itemize}
		\item[\AA1]  \textrm{If} \ $g_j u(1) + (1-g_j) u(0)\leq 0$, then  
        \\$\tau_t(1;j)= \frac{\pi_t(1|\neg j)}{\pi_t(1|j)}$, $\tau_t(0;j)=0$ (under-acceptance),
        \\$\tau_t(1;\neg j)=1$, $\qquad\ \tau_t(0;\neg j)=0$~. 
		\item[\AA2]  \textrm{If} \ $g_j u(1) + (1-g_j) u(0)\geq 0$, then
        \\$\tau_t(1;j)= 1$, $\ \ \tau_t(0;j)=0,$
        \\$\tau_t(1;\neg j)=1$, $\tau_t(0;\neg j)=  \frac{\pi_t(1|j) -\pi_t(1|\neg j)}{1 - \pi_t(1|\neg j)}$ (over-acceptance).
	\end{itemize}}

\begin{proof}
By assumption at time $t$ we have $\pi_t(1|A) \geq \pi_t(1|B)$, from this point on we drop the time subscript for clarity. Let us now express the total utility \eqref{eq:1-step-un} when the demographic parity constraint \eqref{eq:dem-parity-constraint} is satisfied:
	\begin{align*}
	U(\tau) &=g_A \cdot ( u(1) + |u(0|) \cdot ( \pi(1|A) \tau(1;A) - \pi(1|B) \tau(1;B)   )\\
     & + u(1) \cdot \pi(1|B) \tau(1;B) \ - |u(0)| \cdot \pi(0|B) \tau(0;B)
	\end{align*}
	
	First of all we claim that $\pi(1|A) \tau(1;A) \geq \pi(1|B) \tau(1;B)$ for the optimal policy. If this is not the case  then it is possible to increase $\tau(1;A)$ and decrease either $\tau(0;A)$ or $\tau(0;B)$ while still guaranteeing constraint \eqref{eq:dem-parity-constraint} as $\pi(1|A) \geq \pi(1|B)$ all the while obtaining higher utility. 
Moreover, the optimal policy will have $\tau(1;B)=1$  as increasing $\tau(1;B)$ will increase utility as  $\tau(1;A)$ can only increase and $\tau(0;\cdot)$ decrease in response. Similarly $\tau(0;A)=0$ as increasing it will only decrease utility and force $\tau(0;B)$ to increase as a consequence, to meet the demographic parity constraint.
	
	This in turn induces a  lower bound on $\tau(1;A)$, where $\frac{\pi(1|B)}{\pi(1|A)} \leq \tau(1;A)$, next we show that $\tau(1;A)$ can take only the two possible values: $\tau(1;A)=\frac{\pi(1|B)}{\pi(1|A)}$ or $\tau(1;A)=1$.
	
	Suppose $\tau(1;A)=c$, where $c$ is a value between its two possible bounds, then to satisfy the demographic parity constraint \eqref{eq:dem-parity-constraint}:
\[
\pi(1|A) \tau(1;A)  = \pi(0|B) \tau(0;B) + \pi(1|B) 
\] we need:
\[
\tau(0;B)=\frac{c \pi(1|A) -\pi(1|B)}{\pi(0|B)} 
\]	
	Let us compare the utility between the policy with $\tau(1;A)=c$ and that with $\tau(1;A)=\frac{\pi(1|B)}{\pi(1|A)}$. Since by assumption $\tau(1;A)=c$ is utility maximizing, the difference is positive:
	\begin{align*}
	0 \leq & g_A \cdot u(1) \cdot (c- \pi(1|B) \cdot \pi(1|A) ) \cdot \pi(1|A) + (1-g_A) \cdot u(0) \cdot \left(\frac{c \pi(1|A) -\pi(1|B)}{\pi(0|B)}  \right) \cdot \pi(0|B)\\
	&= g_A \cdot u(1) \cdot ( c\pi(1|A) - \pi(1|B)) - g_A \cdot u(0) \cdot (c\pi(1|A) - \pi(1|B)) + u(0) \cdot(c\pi(1|A) - \pi(1|B))\\
	&= (g_A u(1) + g_A |u(0)| +u(0)) \cdot (c\pi(1|A) - \pi(1|B))
	\end{align*} 
	We already have that $c\pi(1|A) - \pi(1|B)\geq0$ from our first claim, and thus for the overall difference of utilities to be positive, then $g_A u(1) + g_A |u(0)| +u(0)\geq0$. However increasing $c$ to $c=1$ will result in higher utility, therefore any interpolation is not optimal.
    
To summarize if $g_A u(1) + g_A |u(0)| +u(0)\geq0$, we have $\tau(1;A)=1$, otherwise  $\tau(1;A)=\frac{\pi(1|B)}{\pi(1|A)}$. As a final note, if instead  $\pi_t(1|A) \leq \pi_t(1|B)$, then it suffices to switch $A$ by $B$ in the conditions and policies. 
\end{proof}
\subsubsection{Continuous Time Dynamics }
In this section we will describe the CT transformation and state some useful lemmas regarding the dynamics in CT.

The transformation to CT is standard, starting from dynamics \eqref{eq:dynamic} and dropping group indices, one can rewrite it as a difference equation:
\begin{equation*}
\pi_{t+1}(1) - \pi_t(1) = \pi_t(1)  \big(f_1(\beta_t(0),\beta_t(1))-1\big) \ + \ \pi_t(0)  f_0(\beta_t(0),\beta_t(1))
\end{equation*}
Now assume that the difference from $t$ to $t + \Delta t$ for all $\Delta t>0$ is proportional to that from $t$ to $t+1$ (linear interpolation):
\begin{align*}
&\pi_{t+\Delta t}\left(1\right) - \pi_t\left(1\right)  =\Delta t \cdot  \big( \pi_t\left(1\right)  \left(f_1\left(\beta_t\left(0\right),\beta_t\left(1\right)\right)-1\right) \ + \ \pi_t\left(0\right)  f_0\left(\beta_t\left(0\right),\beta_t\left(1\right)\right) \big)
\end{align*}
dividing by $\Delta t$:
\begin{equation*}
\frac{\pi_{t+\Delta t}(1) - \pi_t(1)}{\Delta t} =  \pi_t(1)  (f_1(\beta_t(0),\beta_t(1))-1) \ + \ \pi_t(0)  f_0(\beta_t(0),\beta_t(1))
\end{equation*}
Now taking the limit as $\Delta t \to 0$:
\begin{equation*}
\frac{d \pi_t}{dt} =  \pi_t  \big(f_1\left(\beta_t\left(0\right),\beta_t\left(1\right)\right)-1\big) \ + \ \left(1-\pi_t\right)  f_0\left(\beta_t\left(0\right),\beta_t\left(1\right)\right)
\end{equation*}

The CT dynamics \eqref{eq:CTdynamic} thus derived can be motivated in two ways. As is shown next, it maintains the properties of the DT dynamics regarding conditions leading to social equality. Moreover, it possesses additional properties that greatly facilitate the analysis.

\begin{lemma}\label{prop:DT-CTeq}
	Assume the unconstrained policy is implemented in both  discrete time (DT) and continuous time (CT). Assume that the function $f$ defined in equation \eqref{eq:fundynamic} is contractive with constant $L \in [0,1)$, corresponding to assumption \ref{as:fcontractive}, with equilibrium point $\pie$, then the CT dynamics also converges to $\pie$ starting from any initial condition.
\end{lemma}
\begin{proof}
Under the unconstrained policy, the dynamics are an autonomous ordinary differential equation:
\begin{equation*}
\frac{d \pi}{dt} =  \pi(t) \cdot (f_1(0,\pi(t) )-1) \ + \ (1-\pi(t)) \cdot f_0(0,\pi(t) ) 
\end{equation*}

	At any time step $t$, the following will be shown. If $\pi_t<\pie$, $\frac{d \pi}{dt}$ is strictly positive. Otherwise if $\pi_t>\pie$, $\frac{d \pi}{dt}$ is strictly negative and finally if $\pi_t=\pie$ then $\frac{d \pi}{dt}=0$. This implies that the dynamics converges to $\pie$ no matter the starting point.
	
	Indeed, if $\pi_t<\pie$, then $\frac{d \pi}{dt} =  \pi_t \cdot f_1(0,\pi_t \ + \ (1-\pi_t) \cdot f_0(0,\pi_t) - \pi_t = f(\pi_t) - \pi_t$ , where $f(\pi)$ is the DT dynamics. Since the DT dynamics satisfies the following:
	\begin{equation*}
	0 \leq |f(\pie) - f(\pi_t) | \leq L |\pie -\pi_t|
	\end{equation*}
	and using the fact that $ f(\pie)= \pie$:
	\begin{equation}
	0 \leq |\pie - f(\pi_t) | \leq L (\pie -\pi_t) \label{eq:CTpf case1}
	\end{equation}
	From Equation \eqref{eq:CTpf case1} we can see that we must have $f(\pi_t)>\pi_t$ making the derivative positive.
	
	If $\pi_t=\pie$, then $\frac{d \pi}{dt} =  f(\pi_t) - \pi_t = \pie - \pie =0$, and finally if $\pi_t>\pie$ the proof is similar to the case where $\pi_t<\pie$.
\end{proof}

The following shows that populations in CT, unlike in DT, maintain their order.
\begin{lemma} \label{prop:ct_order}
	Following either an unconstrained or affirmative action policy in CT will preserve the initial  advantage. For example, if at $t=0$, $\pi_0(1|B)\leq\pi_0(1|A)$ then for all $t\geq 0$ $\pi_t(1|A)\leq\pi_t(1|B)$.
\end{lemma}
\begin{proof}
	The conservation of advantage will be shown when following each of the noted policies. In all of the following, at time $t=0$: $\pi_0(1|B)\leq \pi_0(1|A)$.
\\ \\
	\textit{\UN:} we remind that the \UN policy is $\tau(1;\cdot)=1$ and $\tau(0;\cdot)=0$. As we are in CT, to have at some $t$, $\pi_t(1|G=B)'>\pi_t(1|G=A)$ at any time $t$ then from a simple continuity argument of the trajectory of $\pi$ over time, there must exist a time point $t'<t$ such that $\pi_{t'}(1|G=A)=\pi_{t'}(1|G=B)$. However, then we must have $\forall t > t'$ that $\pi(t,G=A)'=\pi(t,G=B)$ so that order is conserved.
\\ \\
	\textit{\AA1:} we remind that initially the \AA1 policy is $\tau_0(1;A)= \frac{\pi_0(1|B)}{\pi_0(1|A)}$, $\tau_0(1;B)=1 \ and \ \tau(0;\cdot)=0$. Following the same argument of the \UN case proof, for the groups to change order then there must exist a time-point $t'$ such that $\pi_{t'}(1|G=A)=\pi_{t'}(1|G=B)$, at that point the policy then becomes $\tau_{t'}(1;\cdot)=1$ and $\tau_{t'}(0;\cdot)=0$ and thus the argument is exactly as in the \UN case.
\\ \\
	\textit{\AA2:} finally the \AA2 policy is $\tau_0(1;A)= 1 ,\tau_0(0;A)=0, \tau_0(1;B)=1 \ and \  \tau_0(0;B)=  \frac{\pi_0(1|A) -\pi_0(1|B)}{\pi_0(0|B)}$. Following the same argument for \AA1, at  the time $t'$ the policy becomes: $\tau_{t'}(1;\cdot)=1$ and $\tau_{t'}(0;\cdot)=0$, therefore the argument is exactly as in the \UN case.
\end{proof}

With Lemma \ref{prop:ct_order}, reaching equality with AA is straightforward in CT if it does in DT.

\begin{lemma}\label{lema:CTAAequality}
If an AA policy in DT reaches equality by virtue of contractivity, i.e. $\exists L \in [0,1)$ such that:
\begin{align}\label{eq:AAcontractive}
|f_A(\pi,\pi')-f_B(\pi,\pi')|\leq L|\pi - \pi'| \ \ \forall \pi,\pi' \in [0,1]
\end{align}
 then in CT equality is likewise reached.
\end{lemma}

\begin{proof}
	Any solution $\pi_t$ to the CT dynamics will satisfy the following integral equation \cite{grass2008optimal}:
	\begin{equation*}
	\pi_t = \pi_0 + \int_{0}^{t} \frac{d\pi}{ds} ds
	\end{equation*} 
	Thus let us write the difference between the solutions of the two groups denoted by $\pi$ and $\pi'$ under AA and note that the CT dynamics \eqref{eq:CTdynamic} can be written as:
	\begin{align*}
	\frac{d\pi}{dt} &= f_A(\pi_t,\pi'_t) - \pi_t\\
	\frac{d\pi'}{dt} &= f_B(\pi_t,\pi'_t) - \pi'_t
	\end{align*}
	where $f_A(\pi,\pi')$ and $f_B(\pi,\pi')$ specify the DT dynamics under affirmative action.
	Now let us track the difference at any time $t\geq0$ of the distribution between both groups, denote $\Delta_t := \pi_t-\pi'_t$ and assume that $\pi_0\geq\pi_0'$:
	\begin{equation*}
	\pi_t-\pi'_t = \pi_0 + \int_{0}^{t}( f_A(\pi_s,\pi'_s) - \pi_s) ds - \pi'_0 + \int_{0}^{t} (-f_B(\pi_s,\pi'_s) + \pi'_s )ds
	\end{equation*}
	taking the derivative with respective to $t$:
	\begin{equation*}
	\Delta_t = \Delta_0 + \int_{0}^{t}( f_A(\pi_s,\pi'_s)  -f_B(\pi_s,\pi'_s)) ds -  \int_{0}^{t} \Delta_s ds
	\end{equation*}
	Our assumption that the AA DT policy reaches equality implies that there exists $L \in[0,1)$ such that inequality \eqref{eq:AAcontractive} holds, i.e. $|f_A(\pi_s,\pi'_s)  -f_B(\pi_s,\pi'_s)|\leq L \Delta_s$, and since from lemma \ref{prop:ct_order} order is preserved:
	\begin{align}
	&\Delta_t \leq \Delta_0 + \int_{0}^{t}L \Delta_s ds -  \int_{0}^{t} \Delta_s ds \nonumber\\
	&\frac{d \Delta}{dt} \leq (L-1) \Delta_t \label{eq:AACTequb}
	\end{align}
	Equation \eqref{eq:AACTequb} is obtained by differentiating with respect to $t$. Similarly since $f_A(\pi_s,\pi'_s)  -f_B(\pi_s,\pi'_s)\geq -L\Delta_s$:
	\begin{equation}\label{eq:AACTeqlb}
	\frac{d \Delta}{dt} \geq -(L+1) \Delta_t
	\end{equation}
	From our assumption we know that $L-1<0$ and $\Delta_t\geq0$ from lemma \ref{prop:ct_order} so that $\forall t\geq 0$, $\frac{d \Delta}{dt}\leq0$ . Thus $\Delta_t$ is strictly decreasing over time and is lower bounded by $0$ and when for some $T$, $\Delta_T$ becomes $0$ then this will be the case for all time $t'>T$ by combining the lower bound \eqref{eq:AACTeqlb} and upper bound \eqref{eq:AACTequb}.
\end{proof}

Thus to summarize, transitioning to CT does not break the conditions for DT but in fact facilitates them. One can additionally give a bound on the difference of the distributions between groups at each time step:
\begin{lemma}\label{bound:deltaCT}
	If the DT dynamics is contractive with constant $L \in [0,1)$ as in \eqref{eq:AAcontractive}, then under CT the difference between the distributions of both groups $\Delta_t=\pi_t(1|A)-\pi_t(1|B)$ at any time $t$ obeys:
	\begin{equation}
	\Delta_0 e^{-t(1+L)}\leq \Delta_t \leq \Delta_0 e^{-t(1-L)} \label{eq:CTbound}
	\end{equation}
\end{lemma}

\begin{proof}
	Equation \eqref{eq:AACTequb} gives us an upper bound on the derivative of $\Delta_t$ given the DT dynamics are contractive with Lipschitz constant $L \in [0,1)$:
	
	\begin{equation*}
	\frac{d \Delta}{dt} \leq (L-1) \Delta_t
	\end{equation*}
	assume that $\Delta_t\neq0$  $\forall t\geq0$:
	\begin{equation*}
	\frac{d \Delta}{dt} \frac{1}{\Delta_t } \leq (L-1)
	\end{equation*}
	integrating from $0$ to $t$:
	\begin{equation*}
	\int_{0}^{t}\frac{d \Delta}{ds} \frac{1}{\Delta_s } ds \leq \int_{0}^{t}(L-1)ds
	\end{equation*}
	we get:
	\begin{equation*}
	\ln(\Delta_t) - \ln(\Delta_0) \leq t (L-1) 
	\end{equation*}
	arranging the terms:
	\begin{equation}
	\Delta_t \leq e^{-t(1-L) + \ln(\Delta_0)} \label{eq:CTubound}
	\end{equation}
	If for some $t'\geq0$ we have $\Delta_{t'}=0$, then the conclusion of our theorem will hold for all $t<t'$ with the same proof and for all $t\geq t'$ it trivially holds. Similarly using the lower bound in \eqref{eq:AACTeqlb} we can obtain the lower bound on $\Delta_t$.
\end{proof}

\subsection{Section \ref{sec:unequalityaa}}
\subsubsection{Assumption \ref{as:fcontractive}}
In the beginning of section \ref{sec:unequalityaa}, we made an assumption on the dynamics under the \UN policy to ensure social equality. We now make rigorous why this assumption leads to the appropriate behavior.

First it is necessary that the dynamics has a unique globally attracting equilibrium point:

\begin{proposition}
\label{prop:uniquepiunderequality}
	Unconstrained dynamics, of the form $\pi_{t+1}=f(\pi_t)$ within each group, are equalizing if and only if $f$ has a unique globally attracting equilibrium point $\pie$.
\end{proposition}
\begin{proof}
	(the proof is a simple logical argument) If $f$ has a unique globally attracting equilibrium point, it trivially holds that the dynamics are equalizing. Next, assume the dynamics are equalizing. We prove by contradiction that there can only be a unique globally attracting equilibrium point. Let $\pi_1,\pi_2$ and $\pi_3$ be three distinct points in $[0,1]$. Think of $\pi_1$ as the initial qualified fraction of group $A$. Then, if $\pi_2$ is the initial qualified fraction of group $B$, and if starting at $\pi_1$ and $\pi_2$ respectively the trajectories of $A$ and $B$ were to reach equality, then they must converge over time to a point $\pie_1 \in [0,1]$. Next, think of $\pi_3$ as the initial qualified fraction of group $B$. By the same argument, now the trajectories must converge over time to a (possibly distinct) point $\pie_2 \in [0,1]$, since the definition of equalization doesn't assume a unique limit. However, if $\pie_1$ and $\pie_2$ are different, it implies that the trajectory of $A$ from $\pi_1$ under $f$ can differ, which is a contradiction since $f$ is deterministic.
\end{proof}

Banach's fixed point theorem stated below guarantees that assumption \ref{as:fcontractive} leads to social equality then.

\begin{theorem*} \label{th:banachfp}
	(Banach fixed-point theorem) Let $(X, d)$ be a non-empty complete metric space and let $f : X \to X$ be a contractive mapping. Then $f$ admits a unique fixed-point $x^\mathfrak{e} \in X$ where $x^\mathfrak{e}$ can be found by choosing any arbitrary element $x_0 \in X$ and iterating the function $f$, i.e. define $x_{t}= f^t(x_0)$, then $x_t \to x^\mathfrak{e}$.
\end{theorem*}

To stress again, contractivity is not a necessary condition to reach equality, but it is a simple condition to place on the dynamics that captures the intuition of reaching equality.

More importantly, since the dynamics are defined in terms of the functions $f_0$ and $f_1$, it is more useful to give conditions on the latter to obtain a contractive $f$. For $\pi$ and $\pi'$ such that $\pi - \pi' = \Delta>0$, rewrite the requirement that $|f(\pi) -f(\pi')| \leq L \cdot \Delta$ for $L \in [0,1)$ as:
\begin{align} \label{eq:pre-uncon-lip}
|f(\pi) -f(\pi')|
	&= \left| \ \pi f_1(0,\pi) + f_0(0,\pi) - \pi f_0(0,\pi)   
		 -\ (\pi' f_1(0,\pi') + f_0(0,\pi') - \pi' f_0(0,\pi'))\ \right| \nonumber\\ 
	&= \left| \ \pi (f_1(0,\pi) - f_1(0,\pi-\Delta))
		 + (1-\pi) (f_0(0,\pi) - f_0(0,\pi-\Delta))
	\right. \nonumber \\&\left. \quad+ \Delta (f_1(0,\pi-\Delta) - f_0(0,\pi-\Delta)) \right|  < L\ 
\end{align}

A first step would be to take both $f_1$ and $f_0$ to be Lipschitz. So say
$f_1$ and $f_0$ are $\ell_1$ Lipschitz continuous with constants $L_1$ and $L_0$ respectively. That is $\forall x_1,x_2,y_1,y_2 \in [0,1]$:
\[
	|f_1(x_1,x_2) - f_1(y_1,y_2)| \leq L_1 (|x_1-y_1| + |x_2 -y_2|)
\]
and
\[
	|f_0(x_1,x_2) - f_0(y_1,y_2)| \leq L_0 (|x_1-y_1| + |x_2 -y_2|)~.
\]

This assumption is reasonable as one would expect the jump in the rate of retention ($f_1$) or rate of change ($f_0$) to be bounded as a function of the jump in the selection rates. But looking at Equation \eqref{eq:pre-uncon-lip}, it is not sufficient. How large should $L_0$ and $L_1$ be?
\begin{align}
|f(\pi) -f(\pi')|
	&\leq	|\pi (f_1(0,\pi) - f_1(0,\pi-\Delta))|
	 	+ |(1-\pi) (f_0(0,\pi) - f_0(0,\pi-\Delta))|\nonumber\\
	&\quad 		+ |\Delta (f_1(0,\pi-\Delta) - f_0(0,\pi-\Delta))| \quad (\textrm{triangle inequality}) \nonumber \\
	&<		\Delta  (  \pi L_1 + (1-\pi) L_0 + |f_1(0,\pi-\Delta) - f_0(0,\pi-\Delta)|)
\end{align}

Thus a sufficient condition for $f$ to be a contractive mapping would then be:
\begin{equation} \label{eq:uncon lip}
L_{\UN} \ := \max_{\pi \in [0,1] , \Delta<\pi}  \pi L_1 + (1-\pi) L_0 + |f_1(0,\pi-\Delta) - f_0(0,\pi-\Delta|) <1
\end{equation}

While contractivity is a restrictive condition but it has the obvious benefit of easily describing the rate at which the population reaches equality. In particular, if
\[
	\max_{j\in\{A,B\}} |\pi_0(1|j) - \pie|\leq \Delta~,
\]
then
\[
|\pi_t(1|A) - \pi_t(1|B)| \leq 2\Delta L_{\UN}^t.
\]
Equality is reached at a linear rate (in the exponent): to equalize to within an $\epsilon$ difference, about $\frac{\log \epsilon / \Delta }{\log L_{\UN}}$ time steps are sufficient.

\subsubsection{Preliminary analysis of social equality}

Before proving the main theorems of section \ref{sec:unequalityaa}, we will further discuss an important point in the analysis of affirmative action.

If the institution decides to implement AA, at every time step it solves the linear program \eqref{eq:1-step-un} with the constraint \eqref{eq:dem-parity-constraint}. The solution will take two forms: \AA1 and \AA2 as noted in Proposition \ref{prop:aalp}. The form from one time step to another might change, however this is too unwieldy and complicates the analysis, thus we need to enforce that the solution remains in a one form over all time steps.
Here is a characterization to remain in one case at all times, given with no proof as it follows readily from Proposition \ref{prop:aalp}.

\begin{proposition} \label{prop:one-AA-to-rule-them-all}
Let $\mathcal C_{j,t} := [\pi_t(1|j) \geq \pi_t(1|\neg j)]$ be the $j$-advan\-tage condition. Resolving equality either way, note that $\mathcal C_{\neg j,t} = \neg \mathcal C_{j,t}$. Let $\mathcal C_{-|j} := [g_j u(1) + (1-g_j) u(0)\leq 0]$ be the \AA1 condition under $j$-advantage and let $\mathcal C_{+|j} := [g_j u(1) + (1-g_j) u(0)\geq 0] = \neg \mathcal C_{-|j}$ be the \AA2 condition under $j$-advantage.

Then to remain in \AA1 at all times, it is necessary that for all $t$:
\[
	(\mathcal C_{A,t} \land \mathcal C_{-|A}) \lor (\mathcal C_{B,t} \land \mathcal C_{-|B}),
\]
and in particular it suffices that $\mathcal C_{-|A} \land \mathcal C_{-|B}$. Similarly, to remain in \AA2 at all times, it is necessary that for all $t$:
\[
	(\mathcal C_{A,t} \land \mathcal C_{+|A}) \lor (\mathcal C_{B,t} \land \mathcal C_{+|B}),
\]
and in particular it suffices that $\mathcal C_{+|A} \land \mathcal C_{+|B}$.
\end{proposition}

Now recall that the dynamics under AA are written as:
\begin{align*}
&\pi_{t+1}(1|A) = f_A(\pi_{t}(1|A),\pi_{t}(1|B)) \\
&\pi_{t+1}(1|B) = f_B(\pi_{t}(1|A),\pi_{t}(1|B))~,
\end{align*}
where $f_A$ and $f_B$ are two components of a joint dynamics function $f:[0,1]\times[0,1]\to [0,1]\times [0,1]$. One says that given dynamics of this form, equality will be reached if:
\[ \text{as} \ t\to \infty, \ \left| f^t_A(\pi_{0}(1|B),\pi_{0}(1|A))-f^t_B(\pi_{0}(1|B),\pi_{0}(1|A))\right| \to 0 , \]
for all $\pi_0(1|A)$ and $\pi_0(1|B)$.

Similarly to the unconstrained case, since this would otherwise give no guarantees on the speed of convergence, one can require the joint $(f_A, f_B)$ to satisfy the following contraction condition. There exists $L \in [0,1)$ such that f satisfies:
\begin{align*}
|f_A(\pi,\pi')-f_B(\pi,\pi')|\leq L|\pi - \pi'| \ \ \forall \pi,\pi' \in [0,1]
\end{align*}
This does not quite fit the classical Banach fixed-point theorem, but it is easy to show sufficiency. That is, if the dynamics satisfy this condition then the population will reach equality. Indeed, let $\pi_0,\pi_0' \in [0,1]$:
\begin{align*}
&|f_A(\pi_0,\pi_0')-f_B(\pi_0,\pi'_0)|\leq L|\pi_0 - \pi_0'| \\
&|f_A(\pi_1,\pi_1')-f_B(\pi_1,\pi'_1)|\leq L|\pi_1 - \pi_1'| \leq L^2 |\pi_0 - \pi_0'| \\
& \hspace{4.25cm} \vdots \\
&|f_A(\pi_{t-1},\pi_{t-1}')-f_B(\pi_{t-1},\pi'_{t-1})|\leq L^t|\pi_0 - \pi_0'|
\end{align*}
Since $L<1$ and $|\pi_0 - \pi_0'|\leq 1$ then $\lim_{t \to \infty} L^t = 0$ and  $|f_A(\pi_{t-1},\pi_{t-1}')-f_B(\pi_{t-1},\pi'_{t-1})| \geq 0 \ \ \forall t$ therefore $\lim_{t \to \infty} |f_A(\pi_{t},\pi_{t}')-f_B(\pi_{t},\pi'_{t})| = 0$ and the population will reach equality.

As in the unconstrained case, this contractivity requirement on $f$ can be reduced to sufficient conditions on $f_0$ and $f_1$.

\begin{lemma} \label{lem:AA1condition}
Assume that \AA1 holds at all times, by the necessary or sufficient conditions of Proposition \ref{prop:one-AA-to-rule-them-all}. Then if 
\begin{equation}\label{eq:AA1condition}
L_{\AA1} \ := \max_{\pi \in [0,1]}  |f_1(0,\pi) - f_0(0,\pi)| < 1
\end{equation}
holds, then society equalizes.
\end{lemma}
\begin{proof}
Without loss of generality, assume that $A$ is advantaged. Let $\pi$ be shorthand for $\pi(1|A)$ and $\pi'$ for $\pi(1|B)$. Due to advantage, we can write $\pi'=\pi-\Delta$ for some $\Delta>0$:
\begin{align*}
|f_A(\pi,\pi')-f_B(\pi,\pi')| &=   | \pi f_1(0,\pi - \Delta) + f_0(0,\pi - \Delta) - \pi f_0(0,\pi -\Delta)\\& - (\pi - \Delta) f_1(0,\pi - \Delta) - f_0(0,\pi - \Delta)+ (\pi-\Delta) f_0(0,\pi-\Delta)| \nonumber\\
&= \Delta | f_1(0,\pi - \Delta) - f_0(0,\pi -\Delta)|. \nonumber 
\end{align*}
Since the latter is bounded by $L_{\AA1}$, it follows that the condition of the proposition guarantees the contraction expressed in Equation \eqref{eq:AAcontractive}.
\end{proof}

\begin{lemma} \label{lem:AA2condition}
Assume that \AA2 holds at all times, by the necessary or sufficient conditions of Proposition \ref{prop:one-AA-to-rule-them-all}, and that $f_0$ and $f_1$ are Lipschitz with constants $L_0$ and $L_1$ respectively. Then if 
\begin{equation}\label{eq:AA2condition}
L_{\AA2}:=\max_{0\leq \Delta\leq \pi \leq 1} 2[\pi L_1 +  (1-\pi) L_0] + |f_1(\Delta,\pi-\Delta) - f_0(\Delta,\pi-\Delta)|<1
\end{equation}
holds, then society equalizes.
\end{lemma}
\begin{proof}
Without loss of generality, assume that $A$ is advantaged. Let $\pi$ be shorthand for $\pi(1|A)$ and $\pi'$ for $\pi(1|B)$. Due to advantage, we can write $\pi'=\pi-\Delta$ for some $\Delta>0$:
\begin{align}
|f_A(\pi,\pi')-f_B(\pi,\pi')| 
&=  
|\pi (f_1(0,\pi) - f_1(\Delta,\pi-\Delta))\nonumber \\&+ (1-\pi) (f_0(0,\pi) - f_0(\Delta,\pi-\Delta))\nonumber + \Delta (f_1(\Delta,\pi-\Delta) - f_0(\Delta,\pi-\Delta))| \label{AA2necessarycond}  \\
& \leq\Delta \cdot (  2 \pi L_1 +  2 (1-\pi) L_0 + \nonumber |f_1(\Delta,\pi-\Delta) - f_0(\Delta,\pi-\Delta)|)\nonumber \\& \ \textrm{(triangle inequality + Lipschitz Assumption) } \nonumber
\end{align}
Since the latter is bounded by $L_{\AA2}$, it follows that the condition of the proposition guarantees the contraction expressed in Equation \eqref{eq:AAcontractive}.
\end{proof}

\subsection{Theorem \ref{th:UNimpliesAA1}}
\noindent \paragraph{Theorem \ref{th:UNimpliesAA1}}
\textit{
If equality is reached with an unconstrained (\UN) policy by way of assumption \ref{as:fcontractive}, then it is necessarily reached by an \AA1 policy implemented over all time steps, however with no more, and possibly less, utility at each step.}
\begin{proof}
	The proof of the first part of the theorem is to show that the condition in Lemma \ref{lem:AA1condition} holds. By assumption \ref{as:fcontractive} and  $f$ being continuous and differentiable, then $\forall \pi \in[0,1]$:
	\begin{equation*}
	|f'(\pi)|\leq L_{\UN}
	\end{equation*}
	where $f'$ is the derivative of the dynamics under the \UN policy. Expanding in terms of the proposed dynamics:
	\begin{equation}\label{eq:AA1cond}
	|f_1(0,\pi) -f_0(0,\pi) + \pi f_1'(0,\pi) + (1-\pi) f_0'(0,\pi)|\leq L_{\UN}
	\end{equation}
	where $f_1'$ and $f_0'$ are the partial derivatives of $f_1(0,\pi)$ and $f_0(0,\pi)$ with respect to $\pi$.

Suppose for the sake of contradiction that $|f_1(0,\pi^*) -f_0(0,\pi^*)|=1$ for some $\pi^* \in [0,1]$. Note that since both $f_1$ and $f_0$ take values in $[0,1]$, this implies that one of $f_1(0,\pi^*)$ or $f_0(0,\pi^*)$ is $0$ and the other is $1$. Suppose $\pi^*=1$. If $f_1(0,\pi^*)=1$ it follows that $f_1'(0,\pi^*)\geq0$ and thus $f'(\pi)=1 +f_1'(0,\pi^*) \geq 1$. If instead $f_1(0,\pi^*)=0$ then $f_1'(0,\pi^*)\leq0$ and thus $f'(\pi)=-1 +f_1'(0,\pi^*)\leq1$. Both cases contradict the condition that $|f'(\pi)| \leq L_{\UN} < 1$. The argument is similar for ruling out $\pi^*=0$. Finally, suppose $\pi^* \in (0,1)$. Then one of $f_1$ and $f_0$ attains its maximum at $\pi^*$ and the other its minimum, thus $f_1'(0,\pi^*)=0$ and $f_0'(0,\pi^*)=0$, implying that $|f'(\pi)|=1$. This also contradicts condition \eqref{eq:AA1cond}. It follows that $|f_1(0,\pi^*) -f_0(0,\pi^*)|\neq 1$ and thus $\AA1$ holds.

As the sufficient conditions for both \AA1 and \UN to reach equality are satisfied, in CT the equilibrium point $\pie$  will be reached by both policies by lemma \ref{prop:DT-CTeq}, however their trajectories toward $\pie$ are different and thus each achieves a different utility.

At each time-step $t$, to differentiate between the distributions due to different policies denote by $\pi_t(1|A,\UN)$ the distribution of $\pi(1|A)$ at time $t$ due to the \UN policy starting from $t=0$, and similarly for \AA1. The utility of the \UN policy at time $t$ is:
\[
U_t(\UN) = g_A \cdot u(1) \cdot \pi_t(1|A,\UN) + (1-g_A)\cdot   u(1) \cdot \pi_t(1|B,\UN)
\]
and that of \AA1 at time $t$ is:
\[
U_t(\AA1) = g_A \cdot u(1) \cdot \pi_t(1|B,\AA1) + (1-g_A)\cdot  u(1) \cdot \pi_t(1|B,\AA1)
\]
Consider the difference of the utilities:
\begin{align*}
U_t(\UN) - U_t(\AA1) &= g_A \cdot u(1) \cdot \pi_t(1|A,\UN) + (1-g_A)\cdot  \cdot u(1) \cdot \pi_t(1|B,\UN) \\
&- (g_A \cdot u(1) \cdot \pi_t(1|B,\AA1) + (1-g_A)\cdot  \cdot u(1) \cdot \pi_t(1|B,\AA1))
\end{align*}

Since under \AA1 and \UN the trajectory of $\pi(1|B)$ is identical then $\pi_t(1|B,\AA1) = \pi_t(1|B,\UN)$, thus:
\begin{equation*}
U_t(\UN) - U_t(\AA1) = g_A \cdot u(1) \cdot ( \pi_t(1|A,\UN)- \pi_t(1|B,\UN)
\end{equation*}
Following lemma \ref{prop:ct_order}, one has that $\forall t\geq 0$, $\pi_t(1|A,\UN)- \pi_t(1|B,\UN)\geq 0$, and thus the utility of the \UN policy is always no less than that of \AA1.
\end{proof}

\subsubsection{Theorem \ref{th:aa2moreutility}}
\noindent \paragraph{Theorem \ref{th:aa2moreutility}}
\textit{
Let $\alpha = (1-g_A)u(1)/((1-g_A)u(1) +|u(0)|)$. Given assumptions \ref{as:fcontractive} and \ref{as:aa+contractive} and implementing an \AA2 policy over all time steps, then if $L_{\UN}$ and $L_{\AA2}$ satisfy the following:
\begin{equation*}
L_{\UN}\geq 1 - \alpha
\end{equation*}
and
\begin{equation*}
 L_{\AA2} \leq 1 + ( L_{\UN} -1)/\alpha,
\end{equation*}
\AA2 provides more utility in CT over an infinite time horizon.
}
\begin{proof}
Let us try to look at the difference between the utility at any time step $t$ for the \UN policy and \AA2 policy:
\begin{align*}
U_t(\UN) - U_t(\AA2) &= g_A \cdot u(1) \cdot \pi_t(1|A,\UN) + (1-g_A)\cdot  u(1) \cdot \pi_t(1|B,\UN) \\
&- \big(g_A \cdot u(1) \cdot \pi_t(1|A,\AA2) + (1-g_A)\cdot  ( u(1) \cdot \pi_t(1|B,\AA2) \\&+ u(0) \cdot( \pi_t(1|A,\AA2) - \pi_t(1|B,\AA2))\ ) \big)
\end{align*}
Since under \AA2 and \UN the trajectory of $\pi(1|A)$ is identical then  $\pi_t(1|A,\AA2) = \pi_t(1|A,\UN)$, simplifying things:
\begin{align*}
&U_t(\UN) - U_t(\AA2) = (1-g_A) \cdot \big(  u(1) \cdot  (\pi_t(1|B,\UN) -\pi_t(1|B,\AA2)) - u(0) \cdot (  \pi_t(1|A,\AA2) - \pi_t(1|B,\AA2)   )   \big)
\end{align*}
using the fact that $u(0)<0$ and $u(1)>0$:
\begin{align*}
U_t(\UN) - U_t(\AA2) &= (1-g_A)(u(1) +|u(0)|)\cdot (\pi_t(1|A,\AA2)  - \pi_t(1|B,\AA2) )\\& -  (1-g_A) \cdot u(1) \cdot ( \pi_t(1|A,\UN  -\pi_t(1|B,\UN ))
\end{align*}
In  DT the cumulative utility was the sum over all time steps, similarly in CT the utility will be the integral of the utility function at each step evaluated from $0$ till the final time step.

Now let us look at the total difference in cumulative utility over an infinite time horizon:
\begin{align*}
U(\UN)-U(\AA2)&=(1-g_A)\cdot(u(1) +|u(0)|) \cdot \int_{0}^{\infty}(\pi_t(1|A,\AA2)  - \pi_t(1|B,\AA2) ) dt \\
&-  (1-g_A)\cdot u(1)  \cdot \int_{0}^{\infty}(\pi_t(1|A,\UN)  - \pi_t(1|B,\UN) ) dt 
\end{align*}
Using the two-sided bound on the difference between group distributions of Lemma \ref{bound:deltaCT}, we can lower bound the utility difference:
\begin{align*}
&\leq (1-g_A)\cdot (u(1) +|u(0)|) \cdot \int_{0}^{\infty}  \exp(-t(1-L_{\AA2}) +\ln(\Delta_0) ) dt\\& -  (1-g_A) \cdot u(1)  \cdot \int_{0}^{\infty} \exp(-t(1+L_{\UN}) +\ln(\Delta_0) )  dt\\
& = (1-g_A)\cdot (u(1) +|u(0)|) \frac{\Delta_0}{1-L_{\AA2}}\\& -  (1-g_A)\cdot u(1)  \frac{\Delta_0}{1-L_{\UN}}
\end{align*}

We ask under what conditions of $L_{\AA2}$ and $L_{\UN}$ can the above difference be negative:
\begin{align*}
&(1-g_A)\cdot (u(1) +|u(0)|) \frac{\Delta_0}{1-L_{\AA2}} -  (1-g_A)\cdot u(1) \frac{\Delta_0}{1-L_{\UN}} \leq 0  \\
&  \Leftrightarrow (1-g_A)\cdot (u(1) +|u(0)|) \frac{\Delta_0}{1-L_{\AA2}} \leq (1-g_A)\cdot u(1)  \frac{\Delta_0}{1-L_{\UN}}  \\
&  \Leftrightarrow\frac{1-L_{\UN}}{1-L_{\AA2}} \leq \frac{(1-g_A)\cdot u(1)}{(1-g_A)\cdot u(1) +|u(0)|}\\
& \Leftrightarrow L_{\AA2} \leq L_{\UN} \frac{(1-g_A)\cdot u(1) +|u(0)|} {(1-g_A)\cdot u(1)} + 1-\frac{(1-g_A)\cdot u(1) +|u(0)|}{(1-g_A) \cdot u(1)}  \\
& \Leftrightarrow L_{\AA2} \leq 1 + ( L_{\UN} -1) \cdot \frac{(1-g_A)\cdot u(1) +|u(0)|} {(1-g_A) \cdot u(1)} 
\end{align*}
Since we must have $L_{\AA2}\geq 0$, then we require $L_{\UN}\leq 1 - \frac {(1-g_A)\cdot u(1)} {(1-g_A)\cdot u(1) +|u(0)|}$ as a necessary additional condition.
\end{proof}
\subsection{Section \ref{sec:unequalityaa}}
\subsubsection{Theorem \ref{th:AA1equalitywith-keq}}
\noindent \paragraph{Theorem \ref{th:AA1equalitywith-keq}}
\textit{
Assume status quo bias, i.e. Assumption \ref{status-quo-bias} and a k-equilibrium dynamics under UN, i.e. Assumption \ref{as:keqUN}. Let $j$ be the initially advantaged group. If the disadvantaged group starts at $\pi_0(1|\neg j)\neq \delta_{i}$ for any $i\in \{1,\cdots,k-1\}$, following \AA1 reaches equality in both DT and CT.
}
\begin{proof}
Consider first DT dynamics. By lemma \ref{lem:AA1condition}, it suffices to show that $|f_1(0,\pi)-f_0(0,\pi)|\neq 1$, not for all $\pi$, but only for $\pi \in[0,1]$ \emph{reachable by the disadvantaged group $\neg j$}. In \AA1, recall that the dynamics of group $\neg j$ are the same as in the unconstrained case. Since we assume that initially $\pi_0(1|\neg j)\neq \delta_{i}$ for any $i\in \{1,\cdots,k-1\}$, it follows that the same holds for all time $t$ and none of these $\delta_{i}$ are reachable, because they are unstable equilibria.

Assume then, for the sake of contradiction, that there exists a $\pi^*\in[0,1]$ such that $|f_1(0,\pi^*)-f_0(0,\pi^*)|=1$, from the status-quo bias Assumption \ref{status-quo-bias}, the condition reduces to $f_1(0,\pi^*)-f_0(0,\pi^*)=1$.

This implies that $f_1(0,\pi^*)=1$ and $f_0(0,\pi^*)=0$, thus $f(\pi^*)=\pi^*$, or equivalently $\pi^*$ is an equilibrium of the dynamics $f$. Additionally, by the same arguments as in the proof of Theorem \ref{th:UNimpliesAA1}, we can deduce that if $\pi^*\in\{0,1\}$ then $f'(\pi^*)\geq 1$ and if $\pi^*\in(0,1)$ then $f'(\pi^*)=1$.

It is not possible for $\pi^*$ to be a stable equilibrium, i.e. for all  $i\in [k]$, $\pi^*\neq \pie_{i}$. Otherwise, condition \eqref{cond:k-eqDT2}, which requires that for all $i\in[k]$, $f'(\pie_{i})\leq L_i <1$, is contradicted. The only remaining possibility is for $\pi^*$ to be an unstable equilibrium. But since these are not reachable, we conclude by contradiction that we cannot have $|f_1(0,\pi)-f_0(0,\pi)|=1$.

The CT case follows from Lemma \ref{lema:CTAAequality}.
\end{proof}
\subsubsection{Theorem \ref{th:keqAA1-2}}
\noindent \paragraph{Theorem \ref{th:keqAA1-2}}
\textit{
Under the conditions of Theorem \ref{th:AA1equalitywith-keq}, if the policy is \AA1, then the equalized population generates long-term utility no higher (and possibly lower) than the limiting population under \UN. If the policy is \AA2 and it leads to social equality, then the equalized population generates long-term utility no lower (and possibly higher) than the limiting population under \UN.
}
\begin{proof}
Without loss of generality, let $A$ be the advantaged group. Let $i,j \in [k]$ with $i\leq j$, let $\pi_0(1|B) \in (\delta_{i-1},\delta_{i})$ and  $\pi_0(1|A) \in (\delta_{j-1},\delta_{j})$.
	
Consider first the \AA1 case. Then the policy for group $B$ is identical to the \UN policy as long as $\pi_t(1|B)\leq\pi_t(1|A)$ and from Lemma \ref{prop:ct_order} this is assured. Therefore as $t\to \infty$,  $\pi_t(1|B) \to \pie_{i}$. Theorem \ref{th:AA1equalitywith-keq} implies that $\pi_t(1|A) \to \pi_t(1|B)$ and consequently $\pi_t(1|A) \to \pie_{i}$. The equilibrium state is thus $\lbrace \pie_{i},\pie_{i} \rbrace$. 
	
Under the \UN policy, one ends up with distributions $\lbrace \pie_{i},\pie_{j} \rbrace$ for groups B and A respectively. While this state is possibly unequal however it possesses a higher utility value per time step. Therefore, by reducing selection rates the population is forced to an overall less qualified state.

Consider next the \AA2 case. Then the policy for group $A$ is identical to the \UN policy. Therefore as $t\to \infty$ we have $\pi_t(1|A) \to \pie_{j}$. If the condition for equality under \AA2 is satisfied then $\pi_t(1|B) \to \pi_t(1|A) \to \pie_{j}$. Thus the population converges to the equilibrium state $\lbrace \pie_{j},\pie_{j} \rbrace$. Compared again to the \UN policy, this is is possibly of higher utility value as $i\leq j$. Therefore, by increasing selection rates the population is pushed to be more qualified, but only if \AA2 indeed reaches equality.
\end{proof}

\section{Figure \ref{fig:keqdynamic}} \label{apx:figure-details}
The 3-equilibrium dynamics in figure \ref{fig:keqdynamic} are generated with the following $f_1$ and $f_0$ functions:
\begin{eqnarray*}
f_1(\beta(0),\beta(1))
	&=&	0.5(\beta(1)+\frac{\beta(1)}{5})/1.4 
    	+~\mathrm{e}^{-10^{-9}\cdot\left(\beta(0)+\beta(1)\right)} \cdot \sin\left(18(\beta(0)+\beta(1))\right)+0.1 \\
f_0(\beta(0),\beta(1)) &=& (\beta(1)+\frac{\beta(1)}{5})/1.2 +0.01
\end{eqnarray*}
The resulting function $f$ is shown in figure \ref{fig:keqdynamicfunction}.
\begin{figure}[h]
\centering
  \includegraphics[trim={3.5cm 8cm 3.5cm 8cm},clip,scale=0.7]{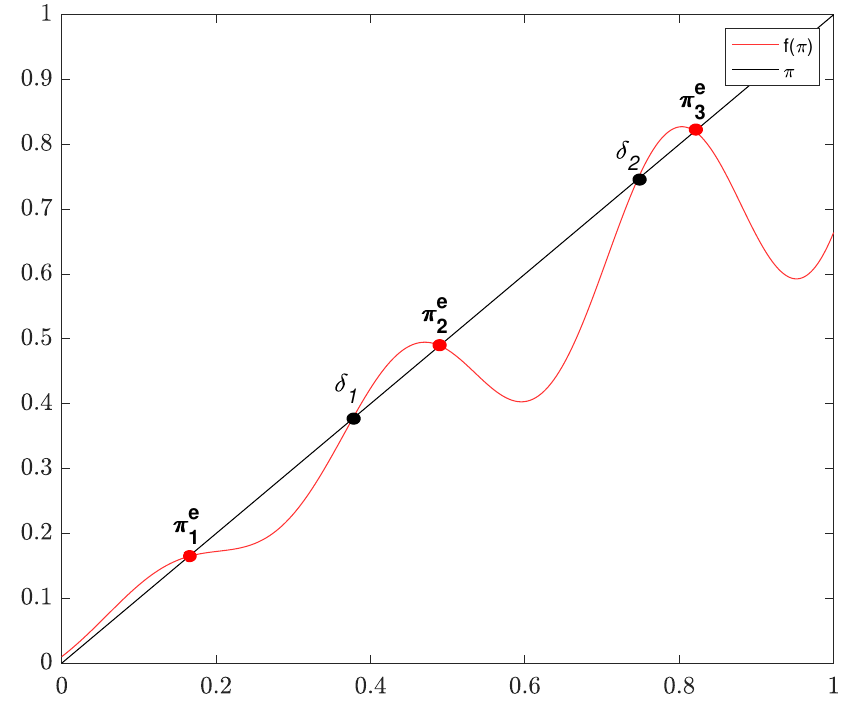}
    \caption{3-equilibrium dynamic function}
\label{fig:keqdynamicfunction}
\end{figure}

\end{document}